\documentclass[twoside,journal,12pt,onecolumn]{IEEEtran}
\usepackage{geometry}
\usepackage{lineno,hyperref}
\usepackage{amsmath, amsthm,mathtools}
\usepackage{amssymb}
\usepackage{amstext}
\usepackage{tikz}
\usepackage{subfig}
\usepackage{multirow}
\usepackage[linesnumbered,ruled]{algorithm2e}
\usepackage{paralist}
\usetikzlibrary{arrows,decorations.pathmorphing,backgrounds,positioning,fit,petri}
\usepgflibrary{decorations.markings} % LATEX and plain TEX and pure pgf
\usepgflibrary[decorations.markings] % ConTEXt and pure pgf
\usetikzlibrary{decorations.markings} % LATEX and plain TEX when using Tik Z
\usetikzlibrary[decorations.markings] % ConTEXt when using Tik Z
\usetikzlibrary{calc,intersections,through,backgrounds}

\modulolinenumbers[5]

\newtheorem{theorem}{Theorem}
\newtheorem{lemma}{Lemma}
\newtheorem{example}{Example}
\newtheorem{corollary}{Corollary}

\newcommand{\alphabet}{\Sigma}
\newcommand{\subseqs}{\mathcal{S}}

\newcommand{\features}{\mathcal{F}}
\newcommand{\LCSs}{\mathcal{L}}

\newcommand{\cov}{\mathcal{C}}
\newcommand{\ihat}{\hat{\imath}}
\newcommand{\subs}{\sqsubseteq}

\newcommand{\rootset}{\mathcal{R}}
\newcommand{\prefixes}{\mathcal{P}}
\newcommand{\postfixes}{\mathcal{Q}}

\newcommand{\A}{\mathcal{A}}
\newcommand{\B}{\mathcal{B}}
\newcommand{\V}{\mathcal{V}}
\newcommand{\D}{\mathcal{D}}

\begin{document}
	
	%\hyphenation{op-tical net-works semi-conduc-tor}
	
	\title{Concordance and the Smallest Covering Set of Preference Orderings}
	\author{Zhiwei Lin   and Hui Wang \\School of Computing  and Mathematics\\
		Ulster University, BT37 0QB, United Kingdom\\
		Email: z.lin@email.ulster.ac.uk and h.wang@ulster.ac.uk\\			
		Cees H. Elzinga\\
		Department of Sociology\\
		VU University Amsterdam and NIDI, The Netherlands\\
		Email:  c.h.elzinga@vu.nl
	}

%	\author{Zhiwei~Lin, Hui~Wang,
%		and~Cees~H.~Elzinga  
%		\thanks{Zhiwei Lin and Hui Wang are   with the School of Computing  and Mathematics , Ulster University, BT37 0QB, UK,  e-mail:  { \{z.lin, h.wang\}@ulster.ac.uk}}% <-this % stops a space
%		\thanks{Cees H. Elzinga is with the Department of Sociology, VU University Amsterdam and NIDI, The Netherlands, 
%			e-mail:   {c.h.elzinga@vu.nl}}% <-this % stops a space
%	}
%	
	\maketitle
	\begin{abstract}
		In   decision making, preference orderings are orderings of a set of items according to the preferences (of {\em judges}). Such orderings arise in a variety of domains, including group decision making and support systems,  consumer marketing, voting and recommendation systems. Measuring the consensus and extracting the consensus  patterns in a set of preference orderings are key to these areas.  In this paper we deal with the representation of sets of preference orderings, the quantification of the degree to which judges agree on their ordering of the items (i.e. the concordance), and the efficient, meaningful description of such sets.
		
		We propose to represent the orderings in a subsequence-based feature space and present a new algorithm to calculate the size of the set of all common subsequences - the basis of a quantification of concordance, not only for pairs of orderings but also for sets of orderings. The new algorithm is fast and storage efficient with a time complexity of only $O(Nn^2)$ for the orderings of $n$ items by $N$ judges and a space complexity of only $O(\min\{Nn,n^2\})$.
		
		Also, we propose to represent the set of all $N$ orderings through a smallest set of covering preferences and present an algorithm to construct this smallest covering set.
	\end{abstract}
	

	%%%%%%%%%%%%%%%%%%%%%%%%%%%%%%%%%%%%%%%%%%%%%%%%%%%%%%%%%%%%%%%%%%%%%%%%%%%%%%%%%%%%%%% Introduction %%%%%%%%%%%%%%%%%%%%%%%%%%%
	\section{Introduction}\label{section:introduction}
	In decision making, preference orderings arise whenever items are ordered  with respect to their relative preference scores. Preference orderings can therefore be used to describe preferences over a set of items. Such orderings exist in a variety of domains, including group decision making and support systems,  consumer marketing, voting and recommendation systems.  For example, in a group decision making system, experts (or judges) use preference orderings to express their preferences over a set of items \cite{Herrera:2002,Herrera:2007,Kwok:2002, Ben-Arieh:2006,Ivan:group:decision:2012,Zhu:2014}. 
	%in the field of ensemble  machine learning, the output of candidate labels from each supervised classifier can be represented as a preference ordering,  and model comparison and evaluation among various classifiers can be done by comparing the output preference orderings \cite{Fagin:2003:similarity:search,Brushe:1998:ml}.
	%in mobile ad hoc and sensor networks, nodes vote on their neighbors to become cluster head and the preferences in this voting process can be modeled  with preference orderings \cite{Ghaderi:2009:ad:hoc,Agarwal:2009:manet:cluster:head};
	
	Formally, let $\Sigma=\{\sigma_1,\sigma_2,\ldots,\sigma_n\}$ denote a set of items, an alphabet, of size $|\Sigma|=n$ and let $\sigma_i\succ\sigma_j$ denote the fact that a judge prefers $\sigma_i$ to $\sigma_j$. Then, given the task of transitively ordering all items from $\Sigma$, the judge will generate a chain of preferences
	\begin{equation*}
	\sigma_{i_1}\succ\sigma_{i_2}\succ\ldots\succ\sigma_{i_n}
	\end{equation*}
	where $i_1,\ldots i_n$ denotes some permutation of $[n]$. Here, we drop the preference ordering relation $\succ$, resulting in  an $n$-long sequence
	\begin{equation*}
	x=x_1\ldots x_n=\sigma_{i_1}\sigma_{i_2}\ldots\sigma_{i_n}
	\end{equation*}
	over $\Sigma$ that represents the preference ordering of some judge $i$. Thus, if $N$ judges each order (the same) $n$ items, a set $X=\{x,y,\ldots\}$ with $|X|=N$ of such preference ordering representing sequences arises. As we assume that the preference orderings are transitive, each item, i.e. each symbol from $\Sigma$, occurs at most once in each sequence. Later, we will relax the assumption that preference orderings are strict and allow for weak orderings, i.e. for a transitive equivalence relation that arises whenever a judge does not prefer either of two items over the other. In such cases, we will say that ``ties'' occur in the orderings.
	In the sequel, we will use the terms ``sequence'', ``ordering sequence'' and ``preference ordering'' as referring to the same concept.
	Most often, when different judges rank the same items according to their preferences, the preference orderings will not fully coincide and some orderings may be the full adverse of other orderings. When analyzing sets of preference orderings, it is convenient to have some quantification of the degree to which the different preference orderings agree or do not agree. Many different quantifications have been proposed \cite{Kendall:1938:rank:correlation,Kendall:1939:ranking,Spearman:1904:measurement,DENUIT:2005:concordance,Taylor:2007:concordance,
		Cees:2011:concordance,Cees:2008:combinatorics,Cees:versatile:2013,Scarsini:1984:concordance} and most of these are only suitable to quantify the concordance between two judges.
	
	%When quantifying concordance between orderings, it is quite intuitive to consider common subsequences. For example, given the sequences $x=abc$, $y=acb$ and $z=cba$, it is intuitive that the concordance between $x$ and $y$, the nonnegative number $C(x,y)$, exceeds $C(x,z)$ and exceeds $C(y,z)$.
	
	A popular quantification of the concordance or similarity between categorical sequences derives from micro-biology and was already proposed in the sixties of the previous century: the so-called edit-distance \cite{Lev66,Got82} and its dual, the length of {\em the longest common subsequences }(for short ``lcs''). The smaller the edit-distance, the longer the lcs and the greater the concordance or similarity between the pertaining sequences. Many different algorithms have been proposed \cite{Hirschberg:1977:LCS,Bergroth:2000:LCS:Survey,Maier:1978:LCS:Complexity,Greenberg:2002:LCS} to calculate the length of the lcs (llcs).
	
	A second, more subtle way to quantify concordance is through the number of {\em all common subsequences } (abbreviated as ``nacs'') instead of only using the lcs. Algorithms to evaluate nacs for pairs of sequences have been proposed in \cite{Cees:2008:combinatorics,Cees:versatile:2013,Hui:acs:ijcai07} and an algorithm to evaluate nacs for sets of orderings has been proposed in \cite{Cees:2011:concordance}.
	
	There are several reasons to prefer nacs to llcs as a measure of concordance. The first reason is that, given two sequences $x$ and $y$, an lcs of $x$ and $y$ may not be a unique sequence. For example, the lcs's of $x=abcd$ and $y=bacd$ are $\{acd,bcd\}$, both satisfying $llcs(x,y)=3$. So, we see that two sets of sequences may have the same llcs while at the same time, one set may have many more distinct lcs's than the other set. In such cases, we would be inclined to consider the set with the most lcs's as the one with the highest concordance. We know that the set of distinct lcs's may be quite big \cite{Elzinga2014b}: the maximum number of $k$-long common subsequences $f(n,k)$ of a pair of $n$-long sequences amounts to
	{\small\begin{equation}\label{fnk}
		f(n,k)=\prod_{i=0}^{k-1}\left\lfloor\frac{n+i}{k}\right\rfloor.
		\end{equation}}
	For example, Equation \eqref{fnk} yields $f(20,7)=1458$. Therefore, quantifying the concordance of a set of orderings through assessing llcs may not be very convincing when the number of lcs's in the one set is much bigger than the same quantity in the other set. These problems do not arise when one uses nacs instead of llcs.
	
	A second reason not to use llcs as a quantification of concordance derives from a general principle that we believe every measure of concordance should adhere to. Let $X$ and $Y$ denote two sets of orderings  and let $C(\cdot)$ denote a measure of concordance. Then $C$ should satisfy the following axiom:
	{\small\begin{equation}\label{Axiom}
		C(X)\geq C(X\cup Y),\text{~~ equality holding iff ~~}Y\subseteq X.
		\end{equation}
	}In case $X\not\subset Y$,  Axiom \eqref{Axiom} states that concordance will never increase by adding more distinct orderings  \cite{Cees:2011:concordance}. So, even small changes in the composition of the pertaining sets will be reflected in the value of $C(\cdot)$. The reader notes that the axiom pertains to \textit{sets}, which means that the multiplicity of certain orderings in a collection or multiset will not affect the concordance in the corresponding set. So, eventual decision making, i.e. the creation of consensus, is separated from the evaluation of concordance.
	\begin{table}
		\caption{lcs's and llcs's of two small sets of orderings, showing that llcs violates the axiom stated in \eqref{Axiom}.}\label{llcstab}
		\begin{center}
			\begin{tabular}{llr}sequences & lcs's & llcs\\\hline
				$X=\{adbc,dacb\}$ & $\{ab,dc,ac\}$ & 2\\
				$Y=\{abcd,cadb\}$ & $\{ab,ad,cd\}$ & 2\\
				$X\cup Y$ & $\{ab\}$ & 2\\\hline
			\end{tabular}
		\end{center}
	\end{table}
	Now consider Table \ref{llcstab}, where we have two sets $X$ and $Y$ with $X\cap Y=\emptyset$. We see that llcs as a measure of concordance fails the axiom \eqref{Axiom} because we have that $llcs(X)=llcs(Y)=llcs(X\cup Y)$. It is not difficult to see that nacs indeed satisfies the axiom embodied in Axiom \eqref{Axiom}.
	Furthermore, llcs only uses part of the information about common subsequences since not all common subsequences are part of an lcs. For example, with $x=abcd$ and $y=adbc$, the common subsequence $ad$ is not contained in the lcs $abc$.
	
	It is therefore clear  that nacs is a preferred quantity to construct a concordance measure from. 
	
	%However, nacs has its drawback too: if, in a set of preference orderings, there is one sequence that is the exact reverse of some other sequence in the set, the set has no common subsequences but the singleton items. For example, if $X=\{abc, bac, cba\}$, $X$ has no common 2-long subsequences since when $u$ is a  subsequence of $abc$, it cannot be a subsequence of $cba$. Clearly, the only subsequences common to all orderings in $X$ are the singletons $a$, $b$ and $c$.
	%
	%So, in practice, we may often encounter sets of preference orderings that are quite homogeneous and yet have very few common subsequences, just because a few of the orderings are (almost) the reverse of some other orderings. Therefore, it is interesting to consider the nacs that might exist in only a (predefined) fraction of the set, as has been the goal of sequential pattern mining \cite{AgrawalSrikant95} for many years. Unfortunately, the Trail-algorithm as proposed in \cite{Cees:2011:concordance} cannot be easily adapted to investigate concordance in only a fraction of the input set.
	
	However useful a measure of concordance may be, it does not explain  what issues, i.e. what subsets of items cause the observed (lack of) concordance. Such insights require a summary description of the preference data  that is  sparse and informative. Thereto, we propose to use the {\em smallest covering set} (SCS for short): the smallest set of orderings to which all common patterns of the data belong. We present an algorithm that constructs precisely this set.
	
	To attain these goals, the paper is structured as follows: in Section 2, we present  the basic concepts and notation that we use in the paper. In Section 3, we discuss the subsequence-based feature space and a generalized kernel to measure its density: the number of common subsequences of all the preference orderings. In Section 4, we present the new algorithm to calculate nacs for pairs of and   sets of sequences and  also discuss tie-handling. In Section 5, we introduce the concept of the smallest covering set as a descriptive tool and an algorithm to construct that set. In Section 6, we summarize, discuss and conclude.
	
	%%%%%%%%%%%%%%%%%%%%%%%%%%%%%%%%%%%%%%%%%%%%%%%%%%%%%%%%%%%%%%%%%%%%%%%%%%%%%%%%%%%%%%% Preliminaries %%%%%%%%%%%%%%%%%%%%%%%%%%
	\section{Preliminaries}\label{section:priminaries}
	This section presents  most of the notation and basic concepts that are used in the paper.

	Let $\alphabet=\{\sigma_1,\ldots,\sigma_{|\Sigma|}\} $ be  an alphabet with $|\alphabet|$ symbols. An $n$-long sequence  $x=x_{1}x_{2}\cdots x_{n}$  over $\alphabet$ is obtained by concatenating $n$ symbols from $\alphabet$, i.e, $x_{i}\in \alphabet$.  The length of $x$ equals the number of symbols in $x$, denoted by $|x|=n$. $\Sigma^*$ denotes the Kleene-star of the alphabet \cite{Sipser2013}, i.e. the set of all finite strings that can be constructed by concatenation from $\Sigma$.
	
	A $k$-long sequence $y=y_1 y_2\cdots y_k$ is a {\em subsequence} of sequence $x$, denoted by $y\sqsubseteq x$, if $y$ can be  obtained by  deleting $|x|-k$, symbols from $x$,  where $k\in[0,|x|]$. For example, let $x=abcac$ and $y=aba$, then obviously,  $aba \sqsubseteq abcac$. Clearly, $cb\not\sqsubseteq x$. Using the boundaries of $k$, we see that $x\sqsubseteq x$ and that there exists an empty sequence $\epsilon \sqsubseteq x$ with $|\epsilon|=0$.  We write $\subseqs(x)$ to denote the set of all non-empty subsequences of $x$. In the rest of the paper, we will be dealing with non-empty subsequences. 
	
	Let $y=y_1 y_2\cdots y_k$ be a subsequence of $x=x_{1}x_{2}\cdots x_{n}$,  $y$ is a {\em substring} of $x$ if there exist two subsequences $u,v\sqsubseteq x$ such that $x=uyv$. We write $x^i$ to denote  the substring $x_1x_2 \cdots x_i$ of $x$ for $i\in [1,n]$.
	
	For any two sequences $x$ and $y$, $z$ is a non-empty {\em common subsequence} of $x$ and $y$ if $z\in\subseqs(x)\cap\subseqs(y)$; we write $z\sqsubseteq(x,y)$ to denote this fact and write $\subseqs(x,y)=\subseqs(x)\cap\subseqs(y)$ for the set of all common non-empty subsequences of $x$ and $y$. We write $\kappa(x,y)=|\subseqs(x,y)|$ to denote the the cardinal of that set.
	
	We use  $\subseqs(x : u) $  to denote the set of all subsequences of $x$ with  suffix  $u$. So, $\subseqs(x:u)$ consists of all subsequences of $x$ that end on $u$. We also write $\subseqs(x, y : u)=\subseqs(x : u) \cap \subseqs(y : u)$, to denote the set of all   common  subsequences with suffix  $u$.

	Let $\ell(x,y)$  (or $\ell$ for short) denote the length of the longest common subsequence of $\subseqs(x,y)$, i.e, $\ell=\max\{ |s| : s \in \subseqs(x,y) \}$.  We also use $\LCSs(x,y)$ to denote the set of all the longest common subsequences of $x$ and $y$, i.e, $\forall z \in \LCSs(x,y), |z|=\ell(x,y)$.

	Analogously, we use $\subseqs(X)$, $\subseqs(X:\sigma)$,  $\LCSs(X)$ and $\ell(X)$ to denote the corresponding quantities for a set  $X$ of sequences, when $|X|\geq 2$.

	The  smallest covering set $\mathcal{C}(X)$ of  $X$ is covering $\subseqs(X)$ if $\forall u,v\in \mathcal{C}(X)$, $u\not\sqsubseteq v$ and $v\not\sqsubseteq u$, and,   $\forall z\in\subseqs(X)$, there exists an $u\in \mathcal{C}(X)$ such that $z\sqsubseteq u$. This amounts to saying  that each common subsequence in $\subseqs(X)$ is a subsequence of at least one sequence in $\mathcal{C}(X)$. For example, let $X=\{abcd,adbc\}$. Then $\subseqs(X)=\{a,b,c,d,ab,ac,ad,bc,abc\}$ and $\cov(X)=\{ad,abc\}$. 
	%Clearly, for all $x\in X$, $\cov(x)=x$ and if $u=lcs(x,y)$, $u\in\cov(x,y)$.
	
	A tie occurs whenever a judge states that $\sigma_i\nsucc\sigma_j$ and $\sigma_j\nsucc\sigma_i$ for  items from $\Sigma$. A tie is interpreted as if a judge cannot decide which of $\sigma_i$ and $\sigma_j$ to prefer. Ties create a partitioning of the alphabet $\Sigma$, such that items from the same part cannot be ordered while elements from different parts are orderable.
	
	%%%%%%%%%%%%%%%%%%%%%%%%%%%%%%%%%%%%%%%%%%%%%%%%%%%%%%%%%%%%%%%%%%%%%%%%%%%%%%%%%%%%%% Concordance Space %%%%%%%%%%%%%%%%%%%%%%
	\section{Concordance in subsequence space} \label{section:concordance:introduction}
	In kernel methods, subsequences are widely used as features to map sequences into higher dimensional spaces, in order to find efficient and effective ways to analyze those sequences \cite{ShaweTaylor_book,Cees:2011:concordance,Cees:2008:combinatorics,Hui:Zhiwei:acs:2007,Hui:acs:ijcai07}.
	Let   $X=\{x,y,\ldots\}$ be a finite set of sequences with $|X|=N$ and let $\features=\features(X)$ denote the set of all subsequences of the sequences of $X$:{\small
		$$\features=\bigcup\limits_{x\in X} \subseqs(x)=\{z_1,z_2,\ldots,z_{|\features(X)|}\}$$}
	We can map any sequence $x\in X$ to a feature vector with features defined by the subsequences in $\features$:
	{\small\begin{equation}
		\phi(x)=\left(f(z_1\sqsubseteq x),f(z_2\sqsubseteq x),\cdots,f(z_{|\features|}\sqsubseteq x)\right)
		\end{equation}
	}Of course, different definitions of the coordinates $f(z_i\sqsubseteq  x)$ lead to different mappings $\phi(\cdot)$ of the feature space \cite{Cees:versatile:2013}. Here, it is convenient to set
	{\small\begin{equation}\label{equation:feature:vector:value}
		f(z_i\sqsubseteq  x) =
		\begin{cases}
		1       & \quad \text{if } z_i\sqsubseteq x\\
		0  & \quad \text{otherwise }\\
		\end{cases}
		\end{equation}}
	since then, the nacs $\kappa(x,y)=|\subseqs(x,y)|$ can be expressed as the inner product of the feature vectors $\phi(x)$ and ${\phi(y)}$:
	{\small \begin{equation}\label{eq:acs:inner:product}
		\kappa(x,y)=\langle \phi(x), \phi(y)\rangle=\sum_{z_i\in \features} f(z_i\sqsubseteq x)f(z_i\sqsubseteq y)
		\end{equation}}
	To generalize to bigger sets of preference orderings, we generalize the inner product to
	{\small \begin{align}
		\kappa(X) &=\langle \phi(x), \phi(y),\cdots,\rangle \notag\\
		&= \sum_{z_i \in \features}\prod_{x\in X} f(z_i\sqsubseteq x )\label{eq:concordance:n:inner:product}
		\end{align}}
	as already proposed in \cite{Cees:2011:concordance}. Properties of this generalized inner product were studied in \cite{Gunawan2002,MANA:n:inner:product}. Clearly, we have that
	{\small$$
		\kappa(X)=|\subseqs(X)|=|\bigcap\limits_{x \in X} \subseqs(x)|
		$$}
	Both $\kappa(x,y)$  and $\kappa(X)$ are not bound from above. Thereto, a straightforward generalization of the cosine similarity is useful:
	{\small$$
		0\leq\hat{\kappa} (X)=\frac{\kappa(X)}{\sqrt[|X|]{\prod_{x\in X} \kappa(x,x)}}\leq1.
		$$
	}%As explained before, if two sequences of $X$ are each others reverts, i.e. if $x,\overset{\leftarrow}{x}\in X$, $\hat{\kappa}(X)\approxeq0$. This may happen even when the other sequences in $X$ are quite homogeneous. In such circumstances, it is wise to express the concordance in a (major) fraction $q$ of the sequences in $X$. 
	
	%Thereto, we define the $q$-concordance by changing the function $g$ in Equation \eqref{gfunc} to
	%\begin{equation}\label{qconcg}
	% g'(X,z_i)=\delta\left(\frac{\sum_{x\in X}f(z_i\sqsubseteq x)}{|X|}\right)\text{ ~~with~~ }\delta(a)=\begin{cases}
	%1       &  \text{if } a\geq q\\
	%0  &  \text{otherwise }\\
	%\end{cases}
	%\end{equation}
	%and
	%\begin{equation}\label{kappaq}
	%  \kappa(X,q)=\sum_{z_i\in\features(X)}g'(X,z_i).
	%\end{equation}
	%Clearly, when $q=1$, $z_i$ must be a subsequence of every sequence in $X$ and then $g(X,z_i)=g'(X,z_i)$, i.e. $\kappa(X,1)=\kappa(X)$. Thus, the quantity $\kappa(X,q)$ equals the number of subsequences that is common to at least a fraction $q$ of the set $X$. The reader notes that, unless $q=1$, not all sequences counted by $\kappa(X,q)$ are necessarily common to one and the same subset of size $q|X|$.
	%\\~\\
	Various types of algorithms have been proposed to calculate $\kappa(x,y)$. In \cite{Hui:acs:ijcai07,Cees:2008:combinatorics}, various dynamic programming algorithms have been proposed and these algorithms all have a time complexity of $O(n^2)$. However, none of these algorithms is easily adaptable to weighting the subsequences according to properties like length, the presence and size of gaps, duration or run-lengths or weighting of properties of the symbols of the alphabet. More versatile types of algorithms have been proposed in \cite{Lodhi:2002:string:kernels} and in \cite{Cees:versatile:2013}, adaptable to a broad range of properties of the subsequences, to weighting of the characters of the alphabet and to efficiently handling run-lengths.
	\\
	%In \cite{Cees:2011:concordance}, the Trail-algorithm was proposed to calculate $\kappa(X)$ with $|X|>2$. It has time complexity of $O(\min\{n^3,Nn^2\})$ but, unfortunately, it is hard to adapt it to calculate $\kappa(X,q)$ for $q<1$. In the next section, we will therefore deal with an algorithm that is easily adapted to evaluate $\kappa(X,q)$ for values of $q\in(0,1]$ and whose time complexity is better than that of the Trail-algorithm.
	%%%%%%%%%%%%%%%%%%%%%%%%%%%%%%%%%%%%%%%%%%%%%%%%%%%%%%%%%%%%%%%%%%%%%%%%%%%%%%%%%%%%%%% Evaluating Kappa's %%%%%%%%%%%%%%%%%%%%%
	\section{Evaluating $\kappa(X)$}\label{section:concordance:calculation}
	To  calculate $\kappa(X)$, we begin with the algorithm that calculates $\kappa(x,y)$, a special case of  $\kappa(X)$ when $|X|=2$.
	%%%%%%%%%%%%%%%%%%%%%%%%%%%%%%%%%%%%%%%%%%%%%%%%%%%%%%%%%%%%%%%%%%%%%% Algorithm 1 %%%%%%%%%%%%%%%%%
	\begin{algorithm}[t]\caption{Pseudo-code for Lemma \ref{lemma:zhiwei:hui:acs} to calculate$\kappa(x,y)$ -- the number of all common subsequences in $x$ and $y$} \label{algorithm:zhiwei:hui:acs}
		\KwData{Sequences $x$ and $y$}
		\KwResult{$\kappa(x,y)$ }
		
		\text{$m=|x|, n=|y|$}\;
		\text{Let $M$ and $I$ be $(m+1)$-long arrays}\;

		\For{$i \leftarrow 1$ \KwTo $m$}{ \nllabel{alg:line:trcse1}
			$I[i]=\infty$\;
			\For{$j \leftarrow 1$ \KwTo $n$}
			{\nllabel{algorithm:hui:acs:line:calc}
				
				\If  %(\tcc{ss})
				{$x_i=y_j$ }{
					$I[i]=j$\;
					\text{break};
				}
			}
		}
		$M[0]=1$\;
		\For{$j \leftarrow 1$ \KwTo $m$}{ \nllabel{alg:line:trcse1}
			%	\If{ $I[j]\neq \infty$}{
			$M[j]=0$\;
			%	}
			\If{$I[j]\neq \infty$	}
			{
				\For{$i \leftarrow 0$ \KwTo $j-1$}
				{\nllabel{algorithm:hui:acs:line:calc}
					
					\If{$I[j]>I[i]$ }{
						$M[j]+=M[i]$
					}
				}
			}	
		}
		{return $\sum_{j=0}^{m}M[j]$}\;		
	\end{algorithm}
	%%%%%%%%%%%%%%%%%%%%%%%%%%%%%%%%%%%%%%%%%%%%%%%%%%%%%%%%%%%%%%%%%%%%%%%%%%%%%%%%%%%%%%%%%%%%%%%%%%%%%%%
	\subsection{Calculating $\kappa(x,y)$}
	The set of all common subsequences $\subseqs(x,y)$ can be partitioned into $|\Sigma|$ subsets of sequences that each end on a particular symbol from $\Sigma$ or, equivalently, a particular symbol from the sequence $x$:
	{\small\begin{equation}\label{partsubs}
		\subseqs(x,y)=\bigcup_{j=1}^{|x|}\subseqs(x,y:x_j).
		\end{equation}
	}Since each subsequence of $\subseqs(x,y)$ belongs to precisely one of the parts, we have that
	{\small\begin{equation}\label{partsubsum}
		\kappa(x,y)=\sum_{j=1}^{|x|}|\subseqs(x,y:x_j)|.
		\end{equation}
	}The latter sum would be easy to calculate when we would know how to calculate a particular summand from the previously calculated summands. This would require that we know the value of the first summand beforehand. And indeed, we do:
	{\small\begin{equation}\label{subset1}
		|\subseqs(x,y:x_1)|=\begin{cases}
		1  &  \text{ if } x_1\sqsubseteq y,\\
		0  &  \text{ if } x_1\not\sqsubseteq y,\\
		\end{cases}
		\end{equation}
	}since $x_1$ is the only\footnote{We do not count the empty subsequence $\epsilon$ since it belongs to all sequences and therefore bears no information on concordance.} subsequence of $x$ that ends on $x_1$.  So, we see that it is convenient to know if and where the symbols of $x$ occur in $y$. Therefore, the algorithm starts to create an indicator-array $\ihat(y,x_j)$, $j\in[0,|x|]$:
	{\small\begin{equation}\label{indic}
		\ihat(y,x_j)=\begin{cases}
		k  &  \text{ if } y_k=x_j,\\
		\infty  &  \text{ if } x_j\not\sqsubseteq y.\\
		\end{cases}
		\end{equation}
	}
	It is convenient to have $\ihat(y,x_0)=0$, since we exploit the convention that for any sequence $x$, $x_0=\epsilon$. The procedure that defines the array $\ihat(y,x_j)$ is in the lines 3 - 11 of the pseudo-code of Algorithm \ref{algorithm:zhiwei:hui:acs} and clearly, this part has time complexity $O(n^2)$.
	\\
	Let us now consider $\subseqs(x,y:x_m)$ for some $1<m\leq|x|$. Clearly, the subsequences in this set can be partitioned again:
	{\small\begin{equation}\label{partendxm}
		\subseqs(x,y:x_m)=\{x_m\}\bigcup_{j=1}^{m-1}\subseqs(x,y:x_jx_m).
		\end{equation}
	}The common subsequences that end on $x_jx_m$ can be constructed from all common subsequences that end on $x_j$ by right-concatenating them with $x_m$ if $x_jx_m\sqsubseteq y$ too. The condition $x_jx_m\sqsubseteq y$ is important since when $x_jx_m\not\sqsubseteq y$, common subsequences that end on $x_jx_m$ do not exist and thus $\subseqs(x,y:x_jx_m)=\emptyset$ or, equivalently, $|\subseqs(x,y:x_jx_m)|=0$. So, we rewrite Eq. \eqref{partendxm} as $\subseqs(x,y:x_m)=$
	{\small 
		\begin{equation}\label{partendxm2}
		\{x_m\}\cup\Big\{zx_m:z\in\Big\{\bigcup_{j\in J}\subseqs(x,y:x_j)\Big\}\Big\}
		\end{equation}
	}
	where $J=\{i:(i\leq m-1)\wedge(x_ix_m\sqsubseteq y)\}$.
	From the last equation, it follows that
	{\small \begin{equation}\label{partsum}
		|\subseqs(x,y:x_m)|=1+\sum_{j=1}^{m-1}|\subseqs(x,y:x_j)|\times\tau(x_jx_m\sqsubseteq y)
		\end{equation}
	}wherein $\tau(\cdot)$is a truth-function: $\tau(\cdot)=1$ precisely if the expression in its argument is true and $\tau(\cdot)=0$ otherwise.
	
	So, if we want to calculate $\subseqs(x,y:x_m)$ from its predecessors, we need a practical way of deciding on the value of the truth-function $\tau$, i.e. of deciding whether or not $x_jx_m\sqsubseteq y$. If $x_jx_m\sqsubseteq y$, $x_j$ should precede $x_m$ in $y$ and if this is not the case, $x_jx_m\not\sqsubseteq y$. The required precedence can be derived from the positions of $x_j$ and $x_m$ in $y$: if $\ihat(y,x_j)<\ihat(y,x_m)$, $x_j$ must precede $x_m$. So,
	{\small\begin{equation}\label{truthdef}
		\tau(x_jx_m\sqsubseteq y)=\tau\big(\ihat(y,x_m)-\ihat(y,x_j)>0\big)
		\end{equation}
	}and this yields a calculable expression
	{\small\begin{align}\label{partendxm3}
		& |\subseqs(x,y:x_m)|=1+\sum_{j=1}^{m-1}\Big( |\subseqs(x,y:x_j)| \nonumber\\
		&   \times\tau\big(\ihat(y,x_m)-\ihat(y,x_j)>0\big)\Big)
		\end{align}
	}
	The reader notes that the compound condition on the set-union operator of Equation \eqref{partendxm2} is reflected in the range of the summation operator and the truth-function appearing in Equation \eqref{partendxm3}.
	The above reasoning, embodied in Eqs. \eqref{partsubsum}, \eqref{subset1} and \eqref{partendxm3}, justifies the following lemma
	\begin{lemma}\label{lemma:zhiwei:hui:acs}
		Let $x,y\in\Sigma^*$ be two sequences. Then the number of all common non-empty subsequences of $x$ and $y$ is given by
		{\small\begin{equation}
			\kappa(x,y)=\sum_{m=1}^{|x|}\kappa(x,y:x_m)
			\end{equation}
		}with
		{\small\begin{equation}
			\kappa(x,y:x_1)=\tau\big(|y|+1-\ihat(y,x_1)>0\big)
			\end{equation}
		}and, for $m>1$,
		{\small\begin{align}
			\kappa(x,y:x_m)=&1+\sum_{j=1}^{m-1}\kappa(x,y:x_j) \nonumber\\
			&\times\tau\big(\ihat(y,x_m)-\ihat(y,x_j)>0\big).
			\end{align}
		}\end{lemma}
		\begin{proof}
			By induction.
		\end{proof}
		
		Lemma \ref{lemma:zhiwei:hui:acs} implies an algorithm with $O(n^2)$ time complexity but only  $O(n)$ space complexity, more efficient than dynamic programming approaches   in \cite{Cees:2008:combinatorics,Hui:acs:ijcai07}. The pseudo-code for Lemma \ref{lemma:zhiwei:hui:acs} is presented in Algorithm \ref{algorithm:zhiwei:hui:acs}.

		\begin{algorithm}\caption{Pseudo-code for Theorem  \ref{theorem:zhiwei:hui:concordance} to calculate concordance in $X$ and for  Corollary \ref{corol:psi} to calculate the length of the longest common subsequence  in $X$.} \label{algorithm:zhiwei:hui:concordance}
			\KwData{A set of sequences $X=\{x_1,\cdots, x_N\}$}
			\KwResult{$\kappa(X), \ell(X), \LCSs(X)$ }
			\tcc{Initialization}
			\text{$m=|x_1|$}\;
			\text{$\varphi[i]=0, \psi[i]=0$  for $\forall i\in [m]$}\;
			\text{   $I_{N\times m}=\Big(I[k][j]=\infty\Big)_{N\times m}$}\;
			\text{ $T_{m\times m}=\Big(T[i][j]=0\Big)_{m\times m}$}\;
			
			\For{$j \leftarrow 1$ \KwTo $m$}{ \nllabel{alg:line:theorem:concordance:init1}
				$I[1][j]=j$\;	
				\For{$k \leftarrow 2$ \KwTo $N$}
				{
					$I[k][j]=\infty$\;								
					\For{$i \leftarrow 1$ \KwTo $|x_k|$}
					{\nllabel{algorithm:hui:acs:line:calc}
						
						\If  %(\tcc{ss})
						{$x_{1j}=x_{ki}$ }{
							$I[k][j]=i$\;
							\text{break};
						}
					}
				}
			}
			\For{$j \leftarrow 1$ \KwTo $m$}{ \nllabel{alg:line:theorem:concordance:init2}
				{	
					\For{$i \leftarrow 1$ \KwTo $j$}
					{			
						$T[j][i]=1$\;
						
						\For{$k \leftarrow 2$ \KwTo $N$}
						{
							
							\If  %(\tcc{ss})
							{$I[k][i]>I[k][j]$ or $I[k][j]=\infty$ }{
								$T[j][i]=0$\;
								%						$D[k][j]=???$\;
								\text{break};
							}
						}
					}
				}
			}
			\tcc{End of initialization}	
			$\varphi[1]=\psi[1]=T[1][1]$	 \;
			\For{$j \leftarrow 2$ \KwTo $m$}{ \nllabel{alg:line:theorem:concordance:algorithm}
				$\varphi[j]=T[j][j]\times \big(\sum_{i=1}^{j-1}\varphi[i]\times T[j][i]\big)$\;
				$\psi[j]=T[j][j]\times\big(1+\max\{T[j][i]\times \psi[i] : 1\leq i<j\}\big)$\;
			}
			$\kappa(X)= \sum_{j=1}^{m} \varphi[j]$\;	$\ell=\max_{1\leq j\leq m}{\psi[j]}$\;
			{return $\kappa(X),\ell$}  \;		
		\end{algorithm}
		\subsection{Calculating $\kappa(X)$ }
		To deal with bigger sets of preference orderings, we have to refine our notation: instead of writing $X=\{x,y,\ldots\}$, we now explicitly index the sequences in $X$ by writing $X=\{x_1,x_2,\ldots,x_N\}$ and $x_i=x_{i1}x_{i2}\ldots x_{in}$. Without loss of generality, we compare all sequences $x_i$, $i\in[2,N]$, with sequence $x_1$. Now we first generalize Equation \eqref{partsubsum}:
		{\small\begin{equation}\label{partsubsumX}
			\kappa(X)=\sum_{m=1}^{|x_1|}\kappa(X:x_{1m})
			\end{equation}
		}and Equation \eqref{subset1}:
		{\small
			\begin{equation}\label{subsetX1}
			\kappa(X:x_{11})=\tau\Big(\bigwedge_{i\in[2,N]}x_{11}\sqsubseteq x_i\Big)
			\end{equation}
		}which generalizes Equation \eqref{subset1}. Furthermore, we generalize Equation \eqref{partsum} to
		{\small
			\begin{align}\label{partsumX}
			\kappa(X:x_{1m})|=&1+\sum_{j=1}^{m-1}\kappa(X:x_{1j})\nonumber \\
			&\times\tau\Big(\bigwedge_{i\in[2,N]}x_{1j}x_{1m}\sqsubseteq x_i\Big)
			\end{align}
		}All that is required to make the above expressions calculable is an efficient way to evaluate the truth-functions of Equations \eqref{subsetX1} and \eqref{partsumX}:
		{\small \begin{equation}\label{truth1X}
			\tau\Big(\bigwedge_{i\in[2,N]}x_{11}\sqsubseteq x_i\Big)=\prod_{i=2}^N\tau\big(|x_i|+1-\ihat(x_i,x_{11})>0\big)
			\end{equation}
		}and we write
		{\small\begin{align}\label{truth2X}
			&\tau\Big(\bigwedge_{i\in[2,N]}x_{1j}x_{1m}\sqsubseteq x_j\Big)\nonumber\\
			&=\prod_{i=2}^N\tau\Big(\ihat(x_i,x_{1j})-\ihat(x_i,x_{1m})>0\Big)
			\end{align}
		}
		Therewith, we arrive at
		\begin{theorem}\label{theorem:zhiwei:hui:concordance}
			Let $X=\{x_1,x_2,\ldots,x_N\}$ denote a set of preference orderings. Then the number of all non-empty common subsequences of $X$ is given by
			{\small
				\begin{equation}\label{calcKappaX}
				\kappa(X)=\sum_{m=1}^{|x_1|}\kappa(X:x_{1m}),
				\end{equation}
			}with
			{\small
				\begin{equation}\label{calcstart}
				\kappa(X:x_{11})=\prod_{i=2}^N\tau\Big((|x_i|+1-\ihat(x_i,x_{11}))>0\Big)
				\end{equation}
			}and, for $1<m\leq|x_1|$, $\kappa(X:x_{1m})= 1+$
			{\small\begin{align}\label{recursX}
				\sum_{j=1}^{m-1}\kappa(X:x_{1j})\nonumber
				\prod_{i=2}^N\tau\Big(\big(\ihat(x_i,x_{1j})-\ihat(x_i,x_{1m})\big)>0\Big)
				\end{align}
			}\end{theorem}
			\begin{proof}
				By induction.
			\end{proof}
			Of course, a practical implementation of the algorithm implied by Theorem \eqref{theorem:zhiwei:hui:concordance} requires preprocessing to calculate the products of the truth-functions as appear in the Theorem. Algorithm \ref{algorithm:zhiwei:hui:concordance} shows the pseudo-code for an implementation of Theorem \ref{theorem:zhiwei:hui:concordance}. During the initialization, firstly an $N\times m$ matrix $I=\left(I\right)_{N\times m}$ is build to store the position indicators: $I_{ij}=\ihat(x_i,x_{1j})$. In the second initialization phase, this array will be used in the construction of the matrix $T=\left(T\right)_{m\times m}$ containing the truth-function products. In particular, $T$ is constructed according to the following rules:
			{\small\begin{equation}\label{constructT}
				T_{kj}=
				\begin{cases}
				1 & \text{if }(k=j)~\wedge~(\forall i:x_{1j}\sqsubseteq x_i)\\
				1 & \text{if }(k<j)~\wedge~(\forall i:x_{1k}x_{1j}\sqsubseteq x_i)\\
				0 & \text{otherwise}
				\end{cases}
				\end{equation}
			}Thus, when $T_{jj}=1$, this implies that $x_{1j}$,  the $j^\text{th}$ character of $x_1$, occurs in all other sequences too and when $T_{kj}=1$, this implies that the subsequence $x_{1k}x_{1j}$ occurs in all sequences. We will use this truth-table in the next subsection to find the longest common subsequences (lcs's) and their length, the llcs.\\
			The following example shows how to use Algorithm \ref{algorithm:zhiwei:hui:concordance} and  Theorem \ref{theorem:zhiwei:hui:concordance} to calculate $\kappa(X)$.
			\begin{example}[Theorem \ref{theorem:zhiwei:hui:concordance}]\label{example:theorem:concordance}
				Given    $X=\{x_1=abcde,x_2=abdce,x_3=bdce\}$ , we set $x=x_1$ and  let  $I=(\ihat_{kj})$, where
				{\small\[I=
					\begin{bmatrix}
					1      & 2 & 3 & 4 & 5 \\
					1      & 2 & 4 & 3& 5 \\
					\infty & 1 & 3 & 2 & 4
					\end{bmatrix}
					\quad
					T=
					\begin{bmatrix}
					0 & 0 & 0 & 0 & 0 \\
					0 & 1 & 0 & 0 & 0 \\
					0 & 1 & 1 & 0 & 0 \\
					0 & 1 & 0 & 1 & 0 \\
					0 & 1 & 1 & 1 & 1
					\end{bmatrix}
					\]}
				Table \ref{table:example:theorem:concordance} then shows how to  calculate $\kappa(X:x_{1j})$ and $\kappa(X)$ with the algorithm implied by Theorem \ref{theorem:zhiwei:hui:concordance},
				\begin{table*}[!htbp]\caption{Example of calculating $\kappa(X)$ for $X=\{x_1=abcde,x_2=abdce,x_3=bdce\}$ with Algorithm \ref{algorithm:zhiwei:hui:concordance} }
					\label{table:example:theorem:concordance}
					\centering
					{\begin{tabular}{l|c|p{7cm}||c}
							\hline
							$j$ & $x_{1j}$ & $\kappa(X:x_{1j})=|\subseqs(X:x_{1j})|$ & $\subseqs(X:x_{1j})$\\
							\hline \hline
							$1$ & $a$ & $1\times 0=0$ & $\emptyset$ \\\hline
							$2$ & $b$ & $1\times 1+0\times 0=1$  &  $\{ b \}$\\\hline
							$3$ & $c$ & $1\times 1 + 0\times 0+1\times 1=2$&  \{$c, bc$\}\\\hline
							$4$ & $d$ & $1\times 1 + 0\times 0+1\times 1+2\times 0=2 $ & $\{d,bd\}$    \\\hline
							$5$ & $e$ & $1\times 1 + 0\times 0+1\times 1+2\times 1+2\times 1=6 $& $\{e,be,ce,bce,de,bde\}$ \\
							\hline
						\end{tabular}
					}
				\end{table*}
				Therefore, from Table \ref{table:example:theorem:concordance}, we see that $\kappa(X)=0+1+2+2+6=11$.
			\end{example}

			\subsection{$\kappa(X)$ for preference orderings   with ties}\label{section:ties:orderings}
			When judges are unable to order certain subsets of the items from the alphabet, ties arise: within a ``tie'' the items appearing in it cannot be ordered with respect to each other. Sequences with ties are easily represented through ``bucket strings'': sequences of small non-empty ``buckets'' or ``sets'' of items and the buckets are ordered. A bucket string, generated by the $i^\text{th}$ judge might then look like, for example
			\begin{equation*}
			b_i=b_{i1},\ldots,b_{ik}= \{a,b\}\{c\}\{d,e,f\},
			\end{equation*}
			implying that judge $i$ preferred both  $a$ and $b$ over $c$ but could not order $a$ and $b$. Only minor changes to the algorithms presented so far, suffice to allow for dealing with these bucket strings.
			
			In order to handle such bucket string $b_{i}=b_{i1}\cdots b_{ik}$ with $n$ symbols,  we introduces a labeling sequence $t_{i}=t_{i1}\cdots t_{in}$,   and  for each symbol $\sigma$ in $b_i$, whose corresponding position is $j$ in $t_i$, we let  $t_{ij}=l$  if the symbol   $  \sigma \in b_{il}$. For example,  the bucket string of  $b_i=  \{a,b\}\{c\}\{d,e,f\}$ has its labeling sequence $t_i$:
			\begin{center}
				
				\begin{tabular}{c|rlcrcl}
					\hline
					$b_i$& $\{a$ & $b\}$ & $\{c\}$ & $\{ d$& $e$&$  f ~~\}$\\
					\hline
					$t_i$      & 1   &1 &2 &3 &3 &3 \\
					\hline
				\end{tabular}
				
			\end{center}
			With $t_i$, we can easily rewrite Theorem \ref{theorem:zhiwei:hui:concordance} for  a set of  ordering sequences with ties.
			Here, because of lack of space, we leave these minor changes to the reader.

			\section{The Smallest Covering Set and its Construction}
			Assuming concordance is high enough, it becomes interesting to scrutinize $X$ in some more detail. This may be done by analyzing the density of the vector-space in which the orderings have been represented through the subsequences. Such an analysis would then use the distances between these vectors: given the $\kappa(x,y)$, such distances are easily obtained since $d(x,y)=\sqrt{2^{n+1}-2-2\kappa(x,y)}$ is a Euclidean metric and the averages $\bar{d}_x=\sum_yd(x,y)/(N-1)$ could be used to isolate ``outlier-judges''. Alternatively, one could compute the distances $d\big(\mathbf{c},\phi(x)\big)$ to the centroid $\mathbf{c}$ of the vector-space. The latter method was described in \cite{Cees:2011:concordance,ShaweTaylor_book}.
			
			Another way of analyzing what is common to the preference orderings in $X$, is to create a set of (sub-)sequences that is in some sense ``characteristic'' for this commonality. An obvious candidate for such a set is the set of all longest common subsequences. However, not all common subsequences are part of an lcs and hence it is interesting to discuss and calculate the broader concept of a smallest covering set. As will appear below, the set of all lcs's is a subset of that covering set.
			
			A covering set of $X$ is a set $\V(X)$ of sequences such that if $x\in\subseqs(X)$, then $\exists y\in\V(X)$ such that $x\subs y$. So, a covering set consists of sequences that ``represent'' all that is common to the sequences in the set $X$. However, this definition is so broad that it even allows for $\subseqs(X)$ itself as a covering set. Therefore it is interesting to look at the Smallest Covering Set $\cov(X)$. A covering set that is smallest contains as few of these covering subsequences as possible. Formally, $\cov(X)\subset\subseqs(X)$ such that
			\begin{enumerate}[~~~~ C. 1]
				\item if $x\in\subseqs(X)$, then $\exists z\in \cov(X)$ such that $x\subs z$,
				\item $|\cov(X)|$ is as small as possible.
			\end{enumerate}
			For example, let $X=\{abcde,eadbc,aedbc\}$. Then $\subseqs(X)=\{a,b,c,d,e,ab,ac,ad,bc,abc\}$ and $\cov(X)=\{abc,ad,e\}$.  Every common subsequence of $X$ is also a subsequence of at least one sequence in $\cov(X)$, the sequences in $\cov(X)$ are not subsequences of each other and the number of sequences in $\cov(X)$ cannot be reduced without violating property C1.
			
			In this example, the first element of $\cov(X)$ is $abc$ and since $abc$ is an lcs of $X$, it should be part of $SCS$ because requirement C2 must be satisfied. When two sequences are lcs's of a set of sequences, they cannot be a subsequence of each other, for if they were, one of them would not be longest. Therefore, we must have that \textit{all} lcs's belong to $SCS$. Furthermore, we note that in the above example, both $ad$ and $e$ belong to $\cov(X)$: they are common to all sequences in $X$ and are not a subsequence of each other or a subsequence of the lcs's. So, it appears that $\cov(X)$ consists of all lcs's of $X$ and all common subsequences of $X$ that are not part of an lcs. So, the sequences in the SCS have an unequivocal interpretation and thus, the SCS is a useful analytical tool. We now focus on the problem of generating the set $\cov(X)$.
			
			As already explained, all lcs's of $X$ must be contained in the SCS: $$\LCSs(X)\subseteq \cov(X)\subseteq\subseqs(X).$$ The construction of the SCS therefore starts with the construction of  $\LCSs(X)$. $\cov(X)=\LCSs(X)$ precisely when all sequences in $\subseqs(X)$ are subsequences of at least one lcs in $\LCSs(X)$. But if this is not the case, i.e. when there exist $y\in\subseqs(X)$ such that $\not\exists z\in \LCSs(X)$ with $y\subs z$, we have to construct additional sequences in order to fulfill the coverage requirement C1. These additional sequences must be shorter than the lcs's and perhaps just consist of one single symbol from the alphabet.
			
			Suppose that for some $u\in\subseqs(X)$ we have that this $u$ is not a subsequence of any of the lcs's of $X$. Then $u$ contains at least one symbol $\sigma$ that does not occur in any of the lcs's of $X$. For suppose, on the contrary, that all characters of this $u$ are contained in some lcs and let $u=u_1u_2\ldots u_{|u|}$. Then there must exist sequences $v_1,\ldots v_{|u|+1}\in\subseqs(X)$, possibly empty, such that
			{\small \begin{equation}
				v_1u_1v_2u_2\ldots v_{|u|}u_{|u|}v_{|u|+1}\in \LCSs(X)
				\end{equation}
			}So, $u$ must be contained in at least one $lcs$ of $X$, contrary to our hypothesis. Therefore, this $u$, not occurring in any of the lcs's, must contain at least one symbol that does not occur in any of the lcs's. If we find symbols that do not occur in any of the lcs's, then this is a sure sign that we have to find more sequences to construct the SCS than just the lcs's. To find these sequences, a good starting point is a symbol not occurring in any of the lcs's and that is precisely what the Algorithm \ref{SCSalg} does.
			
			\begin{algorithm}\caption{Returns the smallest  covering set of a set of preference orderings}\label{SCSalg}
				\KwData{A set of preference orderings $X$.}
				\KwResult{$\cov(X) $}
				$n=|x_{1i}|$\;
				$\D=\{x_{1i}: (i\in [n])\wedge (T_{ii}=1)\}$\;
				$\omega(i)=1, \forall  i\in [n] $\;    
				$\A=\LCSs(X)$ \nllabel{alg:cov:lcs}\;
				$\Lambda=\{\sigma:\sigma\subs x\in\A$\}\;
				$\bar{\Lambda}=\D\backslash\Lambda$\;
				\While{$\bar{\Lambda}\neq\emptyset$}
				{
					$\B=\{v=v_1\lambda v_2:(v_1,v_2\in\subseqs(X))\wedge(\lambda\in\bar{\Lambda})\wedge(v\text{ is alap})\}$\;
					$\omega(i)=0 $ for $\lambda= x_{1i}$ \;
					$\A=\A\cup\B$\;
					$\Lambda=\{\sigma:\sigma\subs x\in \A\}$\;
					$\bar{\Lambda}=\D\backslash\Lambda$\;
				}
				\Return{$\cov(X)=\A$}\;
			\end{algorithm}
			
			The algorithm starts by generating the set $\A$ in Line \ref{alg:cov:lcs}. Then it constructs a set $\bar{\Lambda}$ of symbols that do not occur in any of the lcs's $\bar{\Lambda}=\{\sigma\in\D:\sigma\not\subs x\in\A\}$. If this set is not empty, it picks a symbol $\lambda$ from it and then builds a set $\B$ of sequences that contain $\lambda$, are common to $X$ and are as long as possible (``alap''):
			{\small
				\begin{equation}\label{defB}
				\B=\{v=v_1\lambda v_2:(v\in\subseqs(X))\wedge(\lambda\in\bar{\Lambda})\wedge( v\text{ is alap })\}
				\end{equation}
			}Then $\A$ is set to $\A\cup\B$, $\bar{\Lambda}$ is updated and a new $\B$ is constructed, etc. As soon as $\bar{\Lambda}=\emptyset$, the algorithm returns $\cov(X)=\A$.
			In Algorithm \ref{SCSalg}, it is assumed that there are feasible algorithms to construct $\LCSs(X)$ and the set $\B$ as defined in Equation \eqref{defB}. Therefore, we will deal with these two problems in the next two subsections.
			\subsection{Constructing $\LCSs(X)$}
			Let $x\in \LCSs(X)$. Then $x$ cannot be elongated to a sequence that is still common to the sequences in $X$ and it must have a length $|x|=\ell(X)$. On the other hand, if a sequence has length $\ell(X)$, it must belong to $\LCSs(X)$.
			
			Let    $n=|x_1|$. Clearly $\LCSs(X)$ can be partitioned into subsets that are determined by the symbols in $\Sigma$:
			{\small
				\begin{equation}
				\LCSs(X)=\bigcup_{i\in [n]}\LCSs(X:x_{1i})
				\end{equation}
			} These subsets can be constructed by calculating the lengths of the longest common subsequences that end on each of the symbols from $\Sigma$; the longest of these lengths then equals $\ell(X)$. Therefore, we first create an $|x_1|$-long array $\psi=\psi(1),\ldots,\psi(n)$ such that
			\begin{equation}\label{psidef}
			\psi(i)=\max\{|ux_{1i}|:ux_{1i}\in\subseqs(X)\}.
			\end{equation}
			So, $\psi(i)$ equals the length of the longest common subsequence that ends on the symbol $x_{1i}$ and $\max\{\psi(i):i\in[n]\}=\ell(X)$. To calculate the $\psi(i)$, we use the recursion from Corollary \ref{corol:psi} below.
			\begin{corollary}\label{corol:psi}
				Let $X$ denote a set of preference orderings, let $T$ denote the truth-table as defined in Equation \eqref{constructT} and let the array $\psi$ be defined as in Equation \eqref{psidef}. Then $\psi(i)=$
				{\small \begin{equation}\label{recurpsi}
					\begin{cases}
					0 & \text{ if } T_{ii}=0\\
					1+ \max\{0,T_{ij}\cdot\psi(j):1\leq j<i\} &\text{ otherwise}
					\end{cases}
					\end{equation}
				}and  $\ell(X)=\max\{\psi(i) : 1\leq i \leq n \}$
				
			\end{corollary}
			\begin{proof} By induction, using $\psi(1)\leq1$.\end{proof}
			The algorithm implied in Corollary \ref{corol:LCS} has been integrated in Algorithm \ref{algorithm:zhiwei:hui:concordance}. 
			
			Given that we have calculated $\psi$, we can actually construct the set $\LCSs(X)$: we start by picking a symbol $x_{1i}$ such that $\psi(i)$ is maximal. Now we say that $x_{1i}$ is a candidate-lcs which we will elongate until elongation is not possible anymore. Prefixing $x_{1i}$ is appropriate with $x_{1j}$ when all three of $j<i$, $\psi(j)=\psi(i)-1$ and $x_{1j}x_{1i}\in\subseqs(X)$ hold. Once appropriate prefixes have been found, one searches for new appropriate prefixes, etc.
			
			Therefore, we define a set of all possible prefixes for $x_{1i}$
			{\small \begin{equation}
				\prefixes_i=\Big\{j: (1\leq j <i) \wedge \big(\psi(j)=\psi(i)-1\big) \wedge (T_{ij}=1)  \Big\} 
				\end{equation}
			} The idea of this recursive process, to return a set of subsequences,  is formalized by
			{\small \begin{equation}\label{recursLCS}
				\Theta(i,u)=
				\begin{cases}
				\emptyset & \text{if $\omega(i)=0$}\\
				\Big\{ \Theta(j, x_{1j}u):  \forall j \in \prefixes_i \Big\} & \text{if $\big((\omega(i)\neq 0)$} \\
				& \text{$\wedge (\prefixes_i\neq \emptyset)\big)$} \\
				\Big\{ u \Big\}  &\text{otherwise}
				\end{cases}
				\end{equation}
			}where, for reasons to be explained in the next subsection, the recursion in Equation \eqref{recursLCS} includes the testing of an indicator function $\omega(i)$. Here, we assume that $\omega(i)=1$ for all $i\in[n]$; later we will relax this assumption.

			The function $\Theta$ operates on an index-sequence pair $(i,u)$ where $i$ is the index in $x_1$ of the first symbol in $u$. If $u$ can be appropriately prefixed, i.e. according to the  constraints in its definition, it will return a set of new index-sequence pairs that will be tested for their prefixability. If the sequence in its argument cannot be prefixed, it will be returned by $\Theta$. So ultimately, $\Theta$ will return a set of sequences. We use this recursive function for a ``Depth First Search'' \cite{Cormenetal2001} along the branches of the prefix-tree of sequences that constitute the $\LCSs(X)$. We express these ideas in Corollary \ref{corol:LCS}.
			\begin{corollary}\label{corol:LCS}
				Let $X$ denote a set of preference orderings, let the array $\psi$ be defined as in Equation \eqref{psidef} and let the function $\Theta$ be defined as in Equation \eqref{recursLCS}. Then, with
				{\small\begin{equation}\label{rootdef}
					\rootset=\{i:\psi(i)=\ell(X)\},
					\end{equation}
				}we have that
				{\small\begin{equation}
					\LCSs(X)=\{\Theta(i,u):(i\in\rootset)\wedge(u=x_{1i})\}.
					\end{equation}
				}\end{corollary}
				\begin{proof}By induction.\end{proof}
				According to Corollary \ref{corol:LCS}, the construction of $\LCSs(X)$ starts with the root-set $\rootset$ that, with its argument indices, points to the end-symbols of the lcs's, elongates and finally returns $\LCSs(X)$. The algorithm implied by Corollary \ref{corol:LCS} is shown in Algorithm \ref{alg:Theta}.
				Example \ref{exampleLCS} applies Corollary \ref{corol:LCS} to the set of sequences previously used.
				\begin{example}\label{exampleLCS}
					Let $X=\{x_1=abcde,~x_2=abdce,~x_3=bdce\}$. Then
					{\small \begin{equation*}
						T=
						\begin{bmatrix}
						0 & 0 & 0 & 0 & 0 \\
						0 & 1 & 0 & 0 & 0 \\
						0 & 1 & 1 & 0 & 0 \\
						0 & 1 & 0 & 1 & 0 \\
						0 & 1 & 1 & 1 & 1
						\end{bmatrix}
						\text{ , }\psi=(0,1,2,2,3)
						\end{equation*}
					}  and $\rootset=\{5\}$, hence
					\begin{align*}
					\LCSs(X)=&\\\{\Theta(5&,e)\}\\
					=\{\{\Theta(3,ce)\}&,\{\Theta(4,de)\}\}\\
					=\{\{\Theta(2,bce)\}&,\{\Theta(2,bde)\}\}\\
					=\{\{bce\}&,\{bde\}\}\\
					=\{bce&,bde\}.
					\end{align*}
				\end{example}
				\begin{algorithm}\caption{Function $\Theta$ to construct the $\LCSs(X)$ }\label{alg:Theta}
					{	
						\KwData{sequence $x_1$, arrays $\psi,~\omega$, set of integers $\rootset$}
						\KwIn{integer $i$, sequence $u$}
						\KwOut{$\LCSs(X) $}
						$LCS=\emptyset$\;
						\For{$i\in\rootset$}
						{
							$u\leftarrow x_{1i}$\;
							$A=\emptyset$\;
							\For{$j\leftarrow 1$ \KwTo $i$}
							{
								\If{$\big(\psi(j)=\psi(i)-1\big)\wedge(T_{ij}=1)$}
								{
									$v\leftarrow x_{1j}u$\;
									$A\leftarrow A\cup\{\Theta(j,v)\}$\;
								}
							}
							\eIf{$A=\emptyset$}
							{
								$LCS\leftarrow LCS\cup\{u\}$\;
							}{
							$LCS\leftarrow LCS\cup A$\;    
						}
					}
					\Return $LCS$
				}
			\end{algorithm}

			\subsection{From $\LCSs(X)$ to $\cov(X)$}
			Given that the $\LCSs(X)$ is constructed, we now have to find a way to construct the set $\B$ as defined in Equation \eqref{defB}. $\B$ consists of sequences that contain at least one symbol that is not already part of the sequences that have been labeled as belonging to SCS.
			
			The solution is a bit analogous to that of finding all lcs's: we begin with one such symbol, say $x_{1k}$ not occurring in any lcs, find all the longest prefixes through $\Theta(\cdot)$ and then find all the longest postfixes of the results of $\Theta(\cdot)$. All combinations of such a postfix and a prefix will be a sequence that belongs to the SCS as well.
			
			Only, there is one complication. If we construct all common subsequences that contain $x_{1k}\in\bar{\Lambda}$ and that are alap, some of these common subsequences might contain one or more other characters that do not occur in an lcs either, i.e are contained in $\bar{\Lambda}$ too. Let $x_{1m}$ be such a character and suppose that we just constructed all the alap sequences containing $x_{1k}$. When we now start finding all such sequences containing $x_{1m}$, we will inevitably find some that also contain $x_{1k}$ and such alap common subsequences must have been found already. Therefore, we will have to keep track of the symbols in $\Sigma$ that were already dealt with, i.e. for which we already constructed all common subsequences that contain these symbols. To do just that, let $n=|x_1|$,  we define the array $\omega=(\omega_1\ldots,\omega_{n})$ with $\omega(i)=1$ when $x_{1i}$ is still allowed as a symbol in the construction process, otherwise we set $\omega(i)=0$.

			Finding longest postfixes is analogous to finding longest prefixes through $\Theta$. To do just that, we define   
			{\small\begin{equation}
				\postfixes_i=\{j : (i<j\leq n)  \wedge (T_{ji}=1) \}
				\end{equation}
			}to record all possible postfixes after $x_{1i}$
			and define a recursive function $\Upsilon$:
			
			{\small\begin{equation}\label{recursUpsilon}
				\Upsilon(i,u)=
				\begin{cases}
				\emptyset & \text{if $\omega(i)=0$}\\
				\Big\{ \Upsilon(j, ux_{1j}):   \forall j \in \postfixes_i \Big\} & \text{if $(\omega(i)\neq 0)$} \\
				&\text{ $\wedge (\postfixes_i\neq \emptyset)$} \\
				
				\Big\{ u \Big\}  &\text{otherwise}
				\end{cases}
				\end{equation}
			}
			The recursive $\Theta(i,u)$ and $\Upsilon(i,u)$ can be used to obtain $v_1$ and $v_2$, respectively, as shown in Equation \eqref{defB}.  With these two recursive functions, assuming that $\lambda\in\bar{\Lambda}$ occurs at   $i$-th  position in  $x_1$, then we rewrite Equation \eqref{defB} as 
			\begin{equation}\label{eq:recursive:B}
			\B=\Big\{ \Upsilon\Big(i,\Theta(i,u)\Big) : (x_{1i}=\lambda)\wedge (\lambda\in\bar{\Lambda})  \Big\}.
			\end{equation}
			
			With Corollary \ref{corol:LCS} and Equation \eqref{eq:recursive:B}, we illustrate how Algorithm \ref{SCSalg} works with  the calculations implied by Equation \eqref{recursUpsilon} in Example \ref{exampleSCS}:

			\begin{example}\label{exampleSCS}
				We use $X=\{x_1=abcdef,~~x_2=acfbde,~~x_3=abdcfe\}$ as our toy data set and list all its common subsequences:
				\begin{align*}
				\begin{tabular}{ll}
				$\subseqs(X)$=\{ & a,b,c,d,e,f\\
				& ab,ac,ad,ae,af,bd,be,ce,cf,de\\
				& abd,abe,ace,acf,ade,bde\\
				& abde~~\} \\
				\end{tabular}
				\end{align*}
				
				We will now construct $\cov(X)$. First we generate $\LCSs(X)$. Preprocessing yields
				{\small\begin{equation*}
					T=\begin{bmatrix}
					1 \\
					1 & 1 \\
					1 & 0 & 1 \\
					1 & 1 & 0 & 1 \\
					1 & 1 & 1 & 1 & 1 \\
					1 & 0 & 1 & 0 & 0 & 1\\
					\end{bmatrix}
					\text{  and  }\psi=\left(1,2,2,3,4,3\right),
					\end{equation*}
				}which is sufficient for $\Theta$:
				\begin{align*}
				&\Theta(5,e)=\{\Theta(4,de)\}=\{\Theta(2,bde)\}\\
				&=\{\Theta(1,abde)\}=\{abde\}=\LCSs(X).
				\end{align*}
				Then, we conclude that $\bar{\Lambda}=\{c,f\}$ and thus that $\omega=(1,1,1,1,1,1)$. We start processing $c$ (the reader might check that starting with $f$ would make no difference for the final result):
				{\small \begin{equation*}
					\Theta(3,c)=\{\Theta(1,ac)\}=\{ac\}.
					\end{equation*}
				}Next, we evaluate
				{\small\begin{equation*}
					\Upsilon(3,ac)=\{\Upsilon(5,ace),\Upsilon(6,acf)\}=\{ace,acf\}
					\end{equation*}
				}and set $\omega=(1,1,0,1,1,1)$ since all alap subsequences that contain $x_{13}=c$ have been constructed. Finally, we process $f$ and find
				{\small \begin{equation*}
					\Theta(6,f)=\{\Theta(3,cf)\}=\emptyset
					\end{equation*}
				}since $\omega(3)=0$: indeed, we already found $acf$. So, we conclude that $\cov(X)=\{abde,ace,acf\}$. The reader also nodes that the order of applying $\Theta$ or $\Upsilon$ to the elements of $\bar{\Lambda}$, is immaterial.
			\end{example}
			
			\section{Conclusion}\label{section:conclusion}
			Concordance has been quantified in many ways, most of these using only a small fraction of the information available in preference orderings. We proposed to use the nacs as the basis for evaluating concordance: it uses all of the available information, it is a metric similarity \cite{ChenMaZhang09} in case it is applied to pairs of orderings, the complexity of its calculation  is  only of order $O(Nn^2)$ and at the same time provides for the preprocessing that allows for efficient calculation of the Smallest Covering Set. The SCS is a valuable, easy to compute descriptive tool in the analysis of concordance and may help group leaders in creating consensus in group decision making. The algorithms in the paper have been implemented in Python and made available on Github (\url{https://github.com/zhiweiuu/secs}).

			As a descriptive tool for sets of sequences, SCS could be very useful in applications where sequences have repeating symbols: in web browsing where the same page is visited again, in social demography and career analysis where certain events may happen repeatedly and in the analysis of strands of peptides which consist of only a few elementary building blocks. Therefore, we will extend our research to algorithms for bigger sets of sequences with extended runs of the same symbols and to develop further methods and tools for the analysis of the SCS.
			
			\section*{Acknowledgment}
			The research leading to these results has received funding from the European Research Council under the European Unions Seventh Framework Programme (FP/2007-2013)/ERC Grant Agreement n. 324178 (Project: Contexts of Opportunity, PI: Aart C. Liefbroer), and   from the EU Horizon 2020 research and innovation programmme under grant agreement (No 690238) for DESIREE project.

			\bibliographystyle{IEEEtran}
			\bibliography{bib}
			
		\end{document}